%% file: _main.tex
\RequirePackage{fix-cm}
\documentclass[twocolumn]{article}

\usepackage{titling}
\predate{}
\postdate{}

\date{}

\usepackage[utf8]{inputenc}
\usepackage[T1]{fontenc}

\usepackage{lmodern}
\rmfamily 
\DeclareFontShape{T1}{lmr}{b}{sc}{<->ssub*cmr/bx/sc}{}
\DeclareFontShape{T1}{lmr}{bx}{sc}{<->ssub*cmr/bx/sc}{}

\usepackage{microtype}

\usepackage{amsmath,amssymb,amsfonts,amsthm}
\usepackage[shortlabels]{enumitem}
\usepackage[export]{adjustbox}
\usepackage{xcolor}
\usepackage{cite}
\usepackage{mathtools}

\usepackage{authblk}

\usepackage{subcaption}
\theoremstyle{definition} 

\newtheorem{theorem}{Theorem}

\newtheorem{definition}[theorem]{Definition}
\newtheorem{lemma}[theorem]{Lemma}
\newtheorem*{convention*}{Convention}

\def\T/{\ensuremath{\mathcal{T}}}
\def\NegSat/{\texorpdfstring{\textsc{Ne\-ga\-tor-Sat}}{NEGATOR-SAT}}
\def\NegPath/{\texorpdfstring{\textsc{Ne\-ga\-tor-Path}}{NEGATOR-PATH}}
\def\NegStab/{\texorpdfstring{\textsc{Ne\-ga\-tor-Sta\-bil\-i\-ty}}{NEGATOR-STABILITY}}
\def\dreisat/{\textsc{3-Sat}}
\newcommand*{\GrossO}[1]{\ensuremath{\mathcal{O}\left(#1\right)}}
\DeclareMathOperator{\poly}{poly}

\usepackage{calc} 

\newcommand*\phantomas[3][c]{%
\ifmmode
\makebox[\widthof{$#2$}][#1]{$#3$}%
\else
\makebox[\widthof{#2}][#1]{#3}%
\fi
}

\usepackage{booktabs}

\usepackage[hidelinks]{hyperref}

\begin{document}
\author{Eric Hutter}
\author{Mathias Pacher}
\author{Uwe Brinkschulte}
\affil{Institut für Informatik \\ Goethe University Frankfurt, Germany \\ \texttt{\{hutter,pacher,brinks\}@es.cs.uni-frankfurt.de}}
\title{On the Hardness of Problems Involving Negator Relationships \\ in an Artificial Hormone System}
\maketitle

\begin{abstract}
The Artificial Hormone System (AHS) is a self-organizing middleware to allocate tasks in a distributed system.
We extended it by so-called negator hormones to enable conditional task structures.
However, this extension increases the computational complexity of seemingly simple decision problems in the system:
In \cite{HutterISORC2020} and \cite{HutterSENSYBLE2020}, we defined the problems \NegPath/ and \NegSat/ and proved their NP-completeness.
In this supplementary report to these papers, we show examples of \NegPath/ and \NegSat/, introduce the novel problem \NegStab/ and explain \emph{why} all of these problems involving negators are hard to solve algorithmically.
\end{abstract}

\section*{Note}
This is a supplementary report to \cite{HutterISORC2020} and \cite{HutterSENSYBLE2020}, extending the results of both papers.

Although this report includes the most relevant information required for understanding, we nevertheless expect the reader to be familiar with at least one of the aforementioned papers.

\input{Introduction}
\input{ComputationalComplexity}
\input{Examples}
\input{LogicGates}
\input{Conclusion}

\bibliography{Literature}
\bibliographystyle{IEEEtran}
\end{document}

%% file: Introduction.tex
\section{Introduction}
The \emph{Artificial Hormone System} (AHS) \cite{brinkschulteOrganicRealTimeMiddleware2013} is a middleware based on Organic Computing principles to allocate tasks in distributed systems.
The task allocation is performed by realizing closed control loops based on exchange of short digital messages, called \emph{hormones}.

The AHS is self-configuring (meaning it automatically finds a working initial task distribution based on the processing elements (PEs) suitability for the different tasks) and self-healing (in case a PE fails, its tasks are automatically re-assigned to healthy PEs).
Furthermore, by being completely decentralized, it has no single point of failure.

However, it assumes all tasks to be independent.
Thus, in previous works, we introduced the concept of \emph{negator hormones} into the AHS to enable modeling task dependencies.

This allows to e.g. realize the dependency shown in Figure~\ref{fig:ConceptionNegator}: Here, task $T_i$ may only be assigned to a processing element (PE) if task $T_j$ is \emph{not} assigned.
Such dependency can be represented by the tuple $(T_j, T_i)$.

\begin{figure}
	\centering
	\includegraphics[page=1, scale=2]{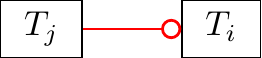}
	\caption{Negator relationship between tasks $T_j$ and $T_i$: If $T_j$ is assigned, $T_i$ must not be assigned}
	\label{fig:ConceptionNegator}
\end{figure}

This report's contribution is four-fold:
\begin{enumerate}
	\item We revisit two decision problems involving negators and present an alternative proof for one problem's NP-hardness.
	\item We introduce an additional decision problem involving negators and prove its NP-completeness.
	\item We simulate instances of these problems in an AHS simulator.
	\item We describe concepts of realizing arbitrary logic circuits using negator relationships, giving an explanation \emph{why} the mentioned problems are NP-complete.
\end{enumerate}

The report is structured as follows:
Section~\ref{sec:CompComp} recaps the problems \NegPath/ and \NegSat/, presents an alternative proof of \NegPath/'s NP-hardness, introduces the novel decision problem \NegStab/ and proves its NP-completeness.
Section~\ref{sec:Examples} shows examples for these problems as well as simulations within an AHS simulator.
Section~\ref{sec:boolLogicWithNegators} shows ways to construct arbitrary logic circuits using negator relationships.
Finally, Section~\ref{sec:Conclusion} concludes this report.

%% file: ComputationalComplexity.tex
\section{Problems Involving Negator Relationships}
\label{sec:CompComp}

In this section, we will examine problems involving negator relationships.
The problems \NegPath/ and \NegSat/ have been analyzed in previous work while \NegStab/ is a novel problem.

\subsection{Recapitulation of previous problems}
We will first recapitulate the previously analyzed problems.

\subsubsection{\NegPath/}
In \cite{HutterISORC2020}, we considered the following problem:
\begin{definition}[\NegPath/]\label{def:NegPath}
Let $\T/$ be a finite set of tasks, $\mathcal{N} \subseteq \T/ \times \T/$ a set of negator relationships among those tasks and
$\mathcal{C} \subseteq \T/ \times \T/$ a set of communication relationships between those tasks
(with $(S,T)\in \mathcal{C}$ iff task $S$ is allowed to send a message to task $T$).

The decision problem \NegPath/ is now stated as follows:
Given two assigned tasks $A, B \in \T/$, does a set of assigned tasks $\mathcal{V}\subseteq \T/$ exist
(with $T \in \mathcal{V}$ iff task $T\in \T/$ is assigned to a PE) so that the following
conditions are all satisfied:
\begin{enumerate}[label=(\arabic*)]
\item There is no task $T \in (\T/ \setminus \mathcal{V})$ that could be assigned to a PE even if all PEs had infinite computational resources,\label{item:NegPath:Def:1}
\item $\mathcal{V}$ is a stable task assignment, i.e. all negator relationships among tasks from $\mathcal{V}$ are respected,\label{item:NegPath:Def:2}
\item $A$ can send a message to $B$ (possibly via multiple hops), i.e. there exists a path from $A$ to $B$ using only edges from $\mathcal{C}$ with
$A$, $B$ and all intermediate tasks contained in $\mathcal{V}$.\label{item:NegPath:Def:3}
\end{enumerate}
\end{definition}
Put in simple terms, \NegPath/ asks whether a stable task assignment exists so that a message sent by task $A$ can be received by task $B$, possibly transformed by other tasks along the way.
Condition~\ref{item:NegPath:Def:1} prevents the system's computational capacities from imposing any limits on such task assignment.

\subsubsection{\NegSat/}
In \cite{HutterSENSYBLE2020}, we looked at a similar problem:
\begin{definition}[\NegSat/]\label{def:NegSat}
Let $\T/$ be a finite set of tasks and $\mathcal{N} \subseteq \T/ \times \T/$ a set of negator relationships among those tasks.

The decision problem \NegSat/ is now stated as follows:
Given a task $A \in \T/$, does a set of assigned tasks $\mathcal{V}\subseteq \T/$ exist
(with $T \in \mathcal{V}$ iff $T$ is assigned to a PE) so that the following
conditions are all satisfied:
\begin{enumerate}[label=(\arabic*)]
\item There is no task $T \in (\T/ \setminus \mathcal{V})$ that could be assigned to a PE even if all PEs had infinite computational resources,\label{item:NegSat:Def:1}
\item $\mathcal{V}$ is a stable task assignment, i.e. all negator relationships among tasks from $\mathcal{V}$ are respected,\label{item:NegSat:Def:2}
\item $A\in \mathcal{V}$, i.e. task $A$ is assigned to some PE.\label{item:NegSat:Def:3}
\end{enumerate}
\end{definition}

In simple terms, \NegSat/ asks whether a stable task assignment exists so that $A$ can be assigned to a PE.
Similarly to \NegPath/, condition~\ref{item:NegSat:Def:1} prevents the system's computational capacities from imposing any limits on such task assignment.

\subsubsection{Computational Complexity}
While both problems seem to be simple at a first glance, they turn out to be quite hard in terms of computational complexity:
\begin{theorem}\label{thm:NegPathNPComplete}
	\NegPath/ is NP-complete.
\end{theorem}
\begin{proof}
	See \cite{HutterISORC2020}.
\end{proof}

\NegPath/ asks for the existence of a communication path and may thus require multiple tasks to be assigned.
Intuitively, asking whether a given \emph{single} task can be assigned at all, seems to be a simpler problem.
However, this is not the case:
\begin{theorem}\label{thm:NegSatNPComplete}
	\NegSat/ is NP-complete.
\end{theorem}
\begin{proof}
	See \cite{HutterSENSYBLE2020}.
\end{proof}

Thus, in terms of computational complexity, both problems are equally hard.
This shows the power introduced by negators:
Unless $\text{P}=\text{NP}$ holds, it is not possible to decide either problem in deterministic polynomial time.

\subsubsection{Alternative Proof for NP-hardness of \NegPath/}
The proof of \NegPath/'s NP-hardness from \cite{HutterISORC2020} used a polynomial-time reduction of \dreisat/ to \NegPath/.
However, when utilizing the NP-completeness of \NegSat/, an elegant and short alternative proof is possible:
\begin{lemma}
	\label{lemma:NegPathNPHard}
	\NegPath/ is NP-hard.
\end{lemma}

The following alternative proof is based on the idea that if an algorithm were known that could decide \NegPath/ in polynomial time, one could use it to also decide \NegSat/ in polynomial time:

\begin{proof}
By reduction of \NegSat/ to \NegPath/:

\begin{figure}
\centering
\includegraphics[max width=\linewidth]{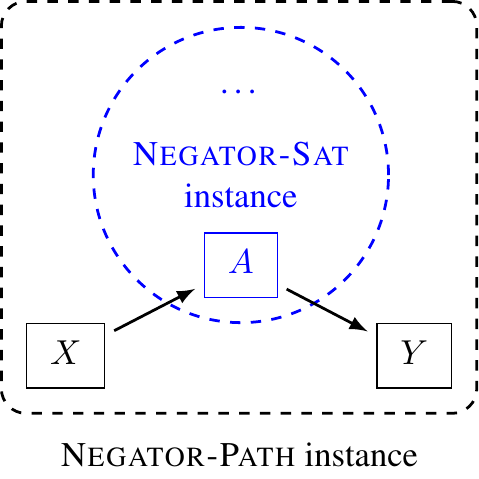}
\caption{Transformation of \NegSat/ to \NegPath/}
\label{fig:NegSatToNegPath}
\end{figure}

Let $(\T/, \mathcal{N}, A)$ be a \NegSat/ instance.
We will transform this input into a \NegPath/ instance as follows:
\begin{itemize}
	\item We introduce two additional tasks $X$ and $Y$ and
	\item introduce the communication relationships $(X, A)$ and $(A, Y)$.
\end{itemize}
It is easy to see that there exists a task assignment of the transformed input so that a path from $X$ to $Y$ exists
iff a task assignment of the original input exists so that task $A$ is assigned (with both assignments satisfying conditions~\ref{item:NegPath:Def:1} and \ref{item:NegPath:Def:2}),
cf. Figure~\ref{fig:NegSatToNegPath}.
Thus, the transformed input is an instance of \NegPath/ iff the original input is an instance of \NegSat/.

The corresponding transformation can be computed in polynomial time w.r.t. the input length:
Only two tasks and two communication relationships have to be added.

Thus, $\NegSat/ \leq_p \NegPath/$ holds.
Since $\NegSat/$ is NP-complete, the NP-hardness of \NegPath/ follows.
\end{proof}

\subsection{\NegStab/}
Informally, both \NegPath/ and \NegSat/ ask whether a stable task assignment exists so that some condition is satisfied simultaneously.

Yet, it turns out that only asking for the existence of a stable task assignment is an NP-complete problem on its own.
Consider the following problem:
\begin{definition}[\NegStab/]
	\label{def:NegStab}
	Let $\T/$ be a finite set of tasks and $\mathcal{N} \subseteq \T/ \times \T/$ a set of negator relationships among those tasks.
	
	The decision problem \NegStab/ is now stated as follows:
	Does a set of assigned tasks $\mathcal{V}\subseteq \T/$ exist (with $T \in \mathcal{V}$ iff $T$ is assigned to a PE) so that the following
	conditions are both satisfied:
	\begin{enumerate}[label=(\arabic*)]
	\item There is no task $T \in (\T/ \setminus \mathcal{V})$ that could be assigned to a PE even if all PEs had infinite computational resources,\label{item:NegStab:Def:1}
	\item $\mathcal{V}$ is a stable task assignment, i.e. all negator relationships among tasks from $\mathcal{V}$ are respected.\label{item:NegStab:Def:2}
	\end{enumerate}
\end{definition}

As with the previous problems, condition~\ref{item:NegStab:Def:1} prevents the system's computational capacities from imposing any limits on the task assignment.

We will now prove that \NegStab/ is NP-hard:
\begin{theorem}
	\label{thm:NegStabNPHard}
	\NegStab/ is NP-hard.
\end{theorem}
\begin{proof}
	By reduction of \NegSat/ to \NegStab/:
	
	Let $I\coloneqq (\T/', \mathcal{N}, A)$ be a \NegSat/ instance.
	We will transform this input into a \NegStab/ instance as follows:
	\begin{itemize}
		\item We introduce three additional tasks $X$, $Y$, $Z$ (so that the resulting task set is $\T/ \coloneqq \T/' \cup \{X,Y,Z\}$) and
		\item introduce the negator relationships $(A,X)$, $(A,Y)$, $(A,Z)$ as well as
		$(X,Z)$, $(Z,Y)$ and $(Y,X)$.
	\end{itemize}
	
	This transformation $\tau$ is sketched in Figure~\ref{fig:NegStabConstruction}.
	It is easy to see that $\tau$ can be computed in polynomial time w.r.t. the input length:
	Only a constant number of tasks and negator relationships have to be added.
		
	\begin{figure}
		\centering
		\includegraphics[max width=\linewidth]{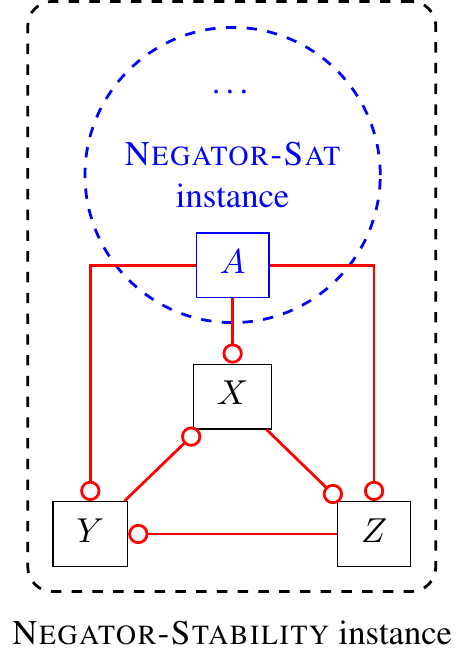}
		\caption{Transformation of \NegSat/ to \NegStab/}
		\label{fig:NegStabConstruction}
	\end{figure}
	
	Thus, it remains to be shown that
	\begin{gather*}
		I \in \NegPath/ \\
		\Updownarrow \\
		\tau(I) \in \NegStab/
	\end{gather*}
	holds.
	
	Before proving this assertion, we will first consider the ``task triangle'' given by tasks $X$, $Y$ and $Z$ as well as their negator relationships:
	It is easy to see that these tasks will never allow a stable system state.
	Instead, these tasks force the assignment to oscillate in some way:
	\begin{enumerate}[label=\arabic*),ref=\arabic*]
		\item Figure~\ref{fig:NegStabOscillationTriangle} shows the oscillation sequence that starts when one task is assigned ($X$ was arbitrarily chosen as the first task here).
			Afterwards, a second task can be assigned since it has no inbound negator
			(per condition~\ref{item:NegStab:Def:1}, this second task will be assigned).
			However, this second task will now send a negator forcing the first task to be stopped, leading to the oscillating sequence shown in the figure.\label{item:NegStabOscVariant1}
		\item Figure~\ref{fig:NegStabOscillationTriangleVariant2} shows a corner case:
			If at least three PEs exist, all three tasks might simultaneously be assigned on three different PEs, leading to the oscillation sequence shown in the figure.\footnote{This oscillation sequence can only occur if all three tasks are assigned simultaneously.
			This requires a very specific AHS constellation to be in effect;
			if less than three tasks are assigned simultaneously, oscillation variant \ref{item:NegStabOscVariant1} occurs instead.}
			\label{item:NegStabOscVariant2}
	\end{enumerate}
	
	\begin{figure*}
		\centering
		\begin{subfigure}{0.25\linewidth}
			\centering
			\includegraphics[page=1]{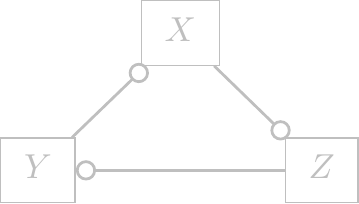}\\
			\textbf{\footnotesize(1)}
		\end{subfigure}%
		\begin{subfigure}{0.25\linewidth}
			\centering
			\includegraphics[page=2]{pics/NegStab/Oscillation}\\
			\textbf{\footnotesize(2)}
		\end{subfigure}%
		\begin{subfigure}{0.25\linewidth}
			\centering
			\includegraphics[page=3]{pics/NegStab/Oscillation}\\
			\textbf{\footnotesize(3)}
		\end{subfigure}%
		\begin{subfigure}{0.25\linewidth}
			\centering
			\includegraphics[page=4]{pics/NegStab/Oscillation}\\
			\textbf{\footnotesize(4)}
		\end{subfigure}\\[2\floatsep]
		\begin{subfigure}{0.25\linewidth}
			\centering
			\includegraphics[page=5]{pics/NegStab/Oscillation}\\
			\textbf{\footnotesize(5)}
		\end{subfigure}%
		\begin{subfigure}{0.25\linewidth}
			\centering
			\includegraphics[page=6]{pics/NegStab/Oscillation}\\
			\textbf{\footnotesize(6)}
		\end{subfigure}%
		\begin{subfigure}{0.25\linewidth}
			\centering
			\includegraphics[page=7]{pics/NegStab/Oscillation}\\
			\textbf{\footnotesize(7)}
		\end{subfigure}%
		\begin{subfigure}{0.25\linewidth}
			\centering
			\includegraphics[page=2]{pics/NegStab/Oscillation}\\
			\textbf{\footnotesize(8)}
		\end{subfigure}
		\caption{Oscillating task assignment, oscillation variant \ref{item:NegStabOscVariant1}}
		\label{fig:NegStabOscillationTriangle}
	\end{figure*}
	
	\begin{figure*}
		\centering
		\begin{subfigure}{0.25\linewidth}
			\centering
			\includegraphics[page=1]{pics/NegStab/Oscillation}\\
			\textbf{\footnotesize(1)}
		\end{subfigure}%
		\begin{subfigure}{0.25\linewidth}
			\centering
			\includegraphics[page=8]{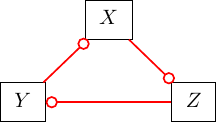}\\
			\textbf{\footnotesize(2)}
		\end{subfigure}%
		\begin{subfigure}{0.25\linewidth}
			\centering
			\includegraphics[page=1]{pics/NegStab/Oscillation}\\
			\textbf{\footnotesize(3)}
		\end{subfigure}%
		\begin{subfigure}{0.25\linewidth}
			\centering
			\includegraphics[page=8]{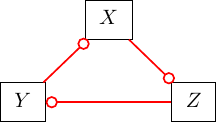}\\
			\textbf{\footnotesize(4)}
		\end{subfigure}%
		\caption{Oscillating task assignment, oscillation variant \ref{item:NegStabOscVariant2}}
		\label{fig:NegStabOscillationTriangleVariant2}
	\end{figure*}
	
	As a result, the transformed input's task assignment $\mathcal{V}$ will only satisfy condition~\ref{item:NegStab:Def:2}
	if task $A$ is assigned and thus prevents $X$, $Y$ and $Z$ from being assigned (and thus from oscillating).
	
	Therefore, we need to distinguish two cases:
	\begin{enumerate}[label=\alph*)]
		\item $I \in \NegSat/$:
			In this case, there exists some set of assigned tasks $\mathcal{V}'\subseteq \T/'$ that satisfies
			conditions~\ref{item:NegSat:Def:1} and \ref{item:NegSat:Def:2} from \NegSat/'s definition (which are identical
			to conditions~\ref{item:NegStab:Def:1} and \ref{item:NegStab:Def:2} from \NegStab/'s definition) and
			includes task $A$.
			
			However, since $A$ is assigned, its negator relationships prevent $X$, $Y$ and $Z$ from being instantiated.
			Thus, $\mathcal{V}'$ is a set of assigned tasks for $\tau(I)$ that satisfies conditions~\ref{item:NegStab:Def:1} and \ref{item:NegStab:Def:2} from \NegStab/:
			No further task can be assigned due to inbound negator relationships and no negator relationship is violated.
			Therefore, $\tau(I)\in \NegStab/$ must hold.
			
			As a result, $I \in \NegSat/ \Rightarrow \tau(I) \in \NegStab/$ holds.
		\item $I \not\in \NegSat/$:
			No set of assigned tasks $\mathcal{V}'\subseteq \T/'$ exists that satisfies
			conditions~\ref{item:NegSat:Def:1} and \ref{item:NegSat:Def:2} from \NegSat/'s definition
			and includes task $A$.
			
			Thus, consider any set of assigned tasks $\mathcal{V}''\subseteq \T/=\T/' \cup \{X,Y,Z\}$.
			
			There are two possibilities for task $A$'s assignment:
			\begin{enumerate}[label=b.\arabic*)]
				\item $A \in \mathcal{V}''$.
					If $A$ is assigned, condition~\ref{item:NegStab:Def:1} and/or \ref{item:NegStab:Def:2} must be violated
					since $I$ is no \NegSat/ instance and thus cannot be assigned without violating at least one of those conditions.
				\item $A \not\in \mathcal{V}''$.
					If $A$ is not assigned, the tasks $X$, $Y$ and $Z$ receive no inbound negator from task $A$ and may thus be assigned.
					However, every possible task assignment of a subset of these three tasks violates condition~\ref{item:NegStab:Def:1} and/or \ref{item:NegStab:Def:2} (cf. Figure~\ref{fig:NegStabOscillationTriangle}):
					If a task receives no inbound negator but is not assigned, condition~\ref{item:NegStab:Def:1} is violated.
					If a task receives an inbound negator and is assigned, condition~\ref{item:NegStab:Def:2} is violated.
			\end{enumerate}
			Thus, no set $\mathcal{V}''\subseteq \T/' \cup \{X,Y,Z\}$ can satisfy both conditions simultaneously.
			Therefore, $\tau(I)$ is no \NegStab/ instance.
			
			Generally, this shows that
			$I \not\in \NegSat/ \Rightarrow \tau(I) \not\in \NegStab/$, which is equivalent to
			$\tau(I) \in \NegStab/ \Rightarrow I \in \NegSat/$.
	\end{enumerate}
	Combining both cases, $I \in \NegSat/ \iff \tau(I) \in \NegStab/$ holds.
	
	As we have already argued that $\tau$ can be computed in polynomial time w.r.t. the input $I$'s length,
	$\tau$ is a polynomial-time reduction from \NegSat/ to \NegStab/.
	
	Thus, \NegStab/ must be at least as hard as \NegSat/.
	Since \NegSat/ is NP-complete (Theorem~\ref{thm:NegSatNPComplete}), the NP-hardness of \NegStab/ follows.
\end{proof}

As \NegPath/ and \NegSat/, \NegStab/ is also NP-complete:
\begin{theorem}
	\NegStab/ is NP-complete.
\end{theorem}
\begin{proof}
	Let $(\T/, \mathcal{N})$ be an input consisting of a set of tasks \T/
	and a set of negator relationships $\mathcal{N}$.
	
	A nondeterministic Turing machine can now nondeterministically select
	a subset $\mathcal{V}\subseteq \T/$ of assigned tasks and then
	deterministically check that conditions~\ref{item:NegStab:Def:1} and \ref{item:NegStab:Def:2}
	from Definition~\ref{def:NegStab} are both satisfied:
	\begin{enumerate}[label=(\arabic*)]
		\item This condition is satisfied if, for each $T \in (\T/ \setminus \mathcal{V})$, there is a negator relationship $(T', T) \in \mathcal{N}$ so that $T' \in \mathcal{V}$.
		This can be checked in $\GrossO(\poly(|\T/|,|\mathcal{N}|))$ time.
		\item This condition is satisfied if, for each $T \in \mathcal{V}$, there is \emph{no} negator relationship $(T', T) \in \mathcal{N}$ so that $T' \in \mathcal{V}$.
		This can also be checked in $\GrossO(\poly(|\T/|,|\mathcal{N}|))$ time.
	\end{enumerate}
	The Turing machine shall accept the input iff all three conditions are satisfied.
	
	Since, after nondeterministically guessing $\mathcal{V}$, the verification can be performed in polynomial time w.r.t. the input length,
	it follows that $\NegStab/ \in \textrm{NP}$.
	Together with Theorem~\ref{thm:NegStabNPHard}, it follows that \NegStab/ is NP-complete.
\end{proof}

As a result, all three problems are equally hard in terms of computational complexity, although this result might appear counter-intuitive:
\NegPath/ asked whether a stable task assignment exists that contains a communication path from $A$ to $B$ (and thus required at least two tasks to be assigned).
\NegSat/ relaxed this problem and only asks whether a stable task assignment exists so that a single distinguished task is assigned.
\NegStab/ further relaxes these conditions and only asks whether a stable task assignment exists at all.
Yet, in terms of computational complexity, all three problems are equally hard and cannot be decided efficiently by a deterministic computer unless $\textrm{P}=\textrm{NP}$ holds.

%% file: Examples.tex
\section{Examples}
\label{sec:Examples}
In this section, we want to give short examples for the transformation from \dreisat/ to \NegPath/ resp. \NegSat/.
However, we only give the results of the transformations and not their theoretical construction.\footnote{
	Note that, for \NegPath/, we use the transformation shown in \cite{HutterISORC2020} and \emph{not} the one used in the alternative proof of Lemma~\ref{lemma:NegPathNPHard}.
}
For more information on this matter, please refer to \cite{HutterISORC2020} and \cite{HutterSENSYBLE2020}.

\subsection{Theoretical Transformation}
In the following, we will use the formula
\begin{equation*}
f\coloneqq (\overline{x_1}\lor \overline{x_2}\lor \overline{x_3})\land (x_1 \lor x_2 \lor x_3).
\end{equation*}
It is easy to see that $f$ is satisfiable and hence $f \in \dreisat/$.

\subsubsection{\NegPath/}
\label{sec:Examples:Theoretical:NegPath}
\begin{figure}
\centering
\includegraphics[max width=\linewidth,page=3]{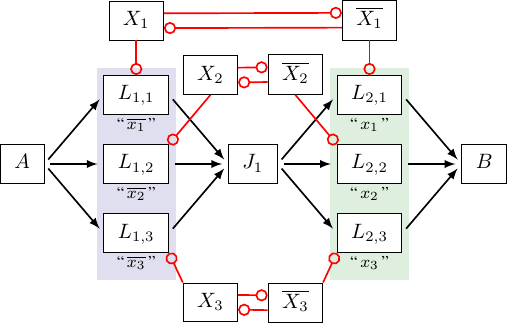}
\caption{\NegPath/ instance constructed from \dreisat/ instance $f=\textcolor{blue!50!black}{(\overline{x_1}\lor\overline{x_2}\lor\overline{x_3})}\land \textcolor{green!50!black}{(x_1\lor x_2\lor x_3)}$}
\label{fig:NegPathExampleConstruction}
\end{figure}

Figure~\ref{fig:NegPathExampleConstruction} shows the construction of a \NegPath/ instance from formula $f$.

\begin{figure*}
\begin{subfigure}{0.48\linewidth}
\centering
\includegraphics[max width=\linewidth,page=5]{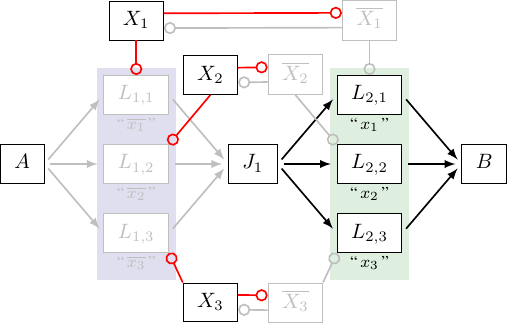}
\caption{Task assignment corresponds to unsatisfying interpretation:\\$\mathcal{I}(x_1)=\mathcal{I}(x_2)=\mathcal{I}(x_3)=1$}
\label{fig:NegPathExampleInterpretedUnsatisfied}
\end{subfigure}
\hfill
\begin{subfigure}{0.48\linewidth}
\centering
\includegraphics[max width=\linewidth,page=4]{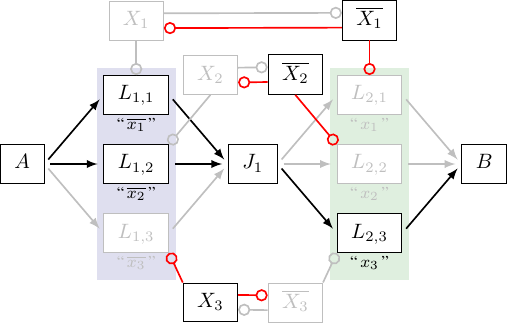}
\caption{Task assignment corresponds to satisfying interpretation:\\$\mathcal{I}(x_1)=\mathcal{I}(x_2)=0$ and $\mathcal{I}(x_3)=1$}
\label{fig:NegPathExampleInterpretedSatisfied}
\end{subfigure}
\caption{\NegPath/ instance from Figure~\ref{fig:NegPathExampleConstruction} with two task assignments. Inactive tasks as well as negator and communication relationships are shown in gray.}
\label{fig:NegPathExampleInterpreted}
\end{figure*}

Figure~\ref{fig:NegPathExampleInterpreted} shows two possible task assignments which are subsets of the task set \T/:
Figure~\ref{fig:NegPathExampleInterpretedUnsatisfied}'s task assignment corresponds to an unsatisfying interpretation of $f$ while
the one corresponding to Figure~\ref{fig:NegPathExampleInterpretedSatisfied} satisfies $f$.

It can easily be checked that the second task assignment $\mathcal{V}$ is consistent with the three conditions stated in Definition~\ref{def:NegPath}:\footnote{
	In fact, the first task assignment satisfies conditions~\ref{item:NegPath:Def:1} and \ref{item:NegPath:Def:2}, but not condition~\ref{item:NegPath:Def:3}.}
\begin{enumerate}[label=(\arabic*)]
	\item For every non-assigned task $T$, there exists a negator relationship $(T',T)\in \mathcal{C}$ with $T' \in \mathcal{V}$
	that prevents $T$ from being assigned,
	\item for each negator relationship $(T,T')\in \mathcal{N}$, $T$ and $T'$ are not simultaneously assigned,
	\item there exists a path from $A$ to $B$ using only tasks in $\mathcal{V}$.
\end{enumerate}
Therefore, $\tau(f)\in\NegPath/$ holds.
Put in simple terms, there exists at least one stable task assignment so that task $A$ can send a message to task $B$ (via task $J_1$ as an intermediary hop).

\subsubsection{\NegSat/}
\label{sec:Examples:Theoretical:NegSat}
\begin{figure}
\centering
\includegraphics[max width=\linewidth,page=3]{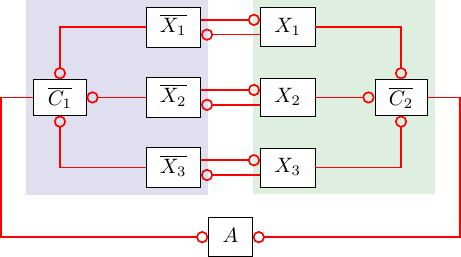}
\caption{\NegSat/ instance constructed from \dreisat/ instance $f=\textcolor{blue!50!black}{(\overline{x_1}\lor\overline{x_2}\lor\overline{x_3})}\land \textcolor{green!50!black}{(x_1\lor x_2\lor x_3)}$}
\label{fig:NegSatExampleConstruction}
\end{figure}

Figure~\ref{fig:NegSatExampleConstruction} shows the construction of a \NegSat/ instance from formula $f$.

\begin{figure*}
\begin{subfigure}{0.48\linewidth}
\centering
\includegraphics[max width=\linewidth,page=4]{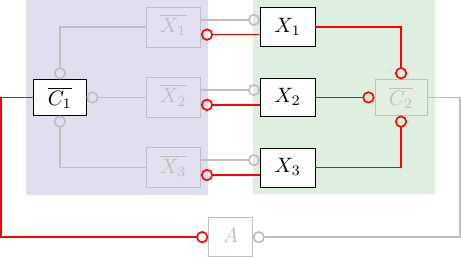}
\caption{Task assignment corresponds to unsatisfying interpretation:\\$\mathcal{I}(x_1)=\mathcal{I}(x_2)=\mathcal{I}(x_3)=1$}
\label{fig:NegSatExampleInterpretedUnsatisfied}
\end{subfigure}
\hfill
\begin{subfigure}{0.48\linewidth}
\centering
\includegraphics[max width=\linewidth,page=5]{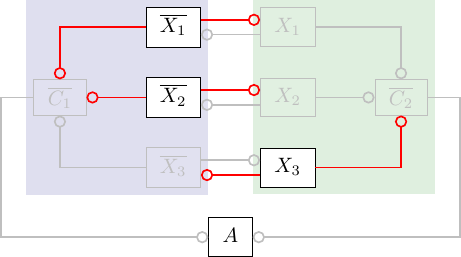}
\caption{Task assignment corresponds to satisfying interpretation:\\$\mathcal{I}(x_1)=\mathcal{I}(x_2)=0$ and $\mathcal{I}(x_3)=1$}
\label{fig:NegSatExampleInterpretedSatisfied}
\end{subfigure}
\caption{\NegSat/ instance from Figure~\ref{fig:NegSatExampleConstruction} with two task assignments. Inactive tasks as well as negator and communication relationships are shown in gray.}
\label{fig:NegSatExampleInterpreted}
\end{figure*}

Two task assignments corresponding to an unsatisfying resp. satisfying interpretation of $f$ are shown in Figure~\ref{fig:NegSatExampleInterpreted}.

It can easily be checked that the task assignment from Figure~\ref{fig:NegSatExampleInterpretedSatisfied} $\mathcal{V}$ is consistent with the three conditions stated in Definition~\ref{def:NegSat}:\footnote{
	In fact, the first task assignment satisfies conditions~\ref{item:NegSat:Def:1} and \ref{item:NegSat:Def:2}, but not condition~\ref{item:NegSat:Def:3}.}
\begin{enumerate}
	\item For every non-assigned task $T$, there exists a negator relationship $(T',T)\in \mathcal{C}$ with $T' \in \mathcal{V}$
	that prevents $T$ from being assigned,
	\item for each negator relationship $(T,T')\in \mathcal{N}$, $T$ and $T'$ are not simultaneously assigned,
	\item $A\in\mathcal{V}$ holds.
\end{enumerate}
Therefore, $\tau(f)\in\NegSat/$ holds.
Put in simple terms, there exists at least one stable task assignment so that task $A$ is assigned.

\subsubsection{\NegStab/}
Since the transformation used for \NegStab/ only involves the addition of three additional tasks and six negator relationships to an \NegSat/ instance, we omitted giving a complete example for brevity.
Please refer to Figure~\ref{fig:NegStabConstruction} instead.

\subsection{In an AHS Simulator}
In order to show that the theoretical constructions used for proving the NP-completeness of \NegPath/ and \NegSat/ can also be realized in practice, we implemented our negator extension in an AHS simulator~\cite{rentelnExaminatingTaskDistribution2008,brinkschulteOrganicRealTimeMiddleware2013}.
This simulator performs an accurate simulation of the AHS' inner workings by exchanging hormones the same way the AHS middleware does.

We then translated the constructions from sections~\ref{sec:Examples:Theoretical:NegPath} and \ref{sec:Examples:Theoretical:NegSat} into appropriate simulator configuration files.
There were two reasons we chose a simulator instead of the real AHS middleware for these scenarios:
\begin{enumerate}
	\item The simulator allows to force an initial task assignment, which is not possible in the AHS middleware.
	This allowed us to easily instantiate one specific assignment task for each variable $x_k$.
	The AHS middleware would have given us no direct control over which assignment tasks would have been instantiated.
	It shall be noted that only the assignment tasks were initially fixed, all other tasks were automatically
	instantiated by the simulator respecting all negator relationships.
	\item Using the simulator, it is easy to restrict the assignment of certain tasks to specific PEs.
	We made use of this feature in two ways:
	\begin{enumerate}
		\item For each assignment task $X_k$ resp. $\overline{X_k}$, we created one PE that can only instantiate this single task.
		This allowed us to easily switch the corresponding interpretation of the formula by temporarily suspending the PE.
		As a result, the assignment task $X_k$ resp. $\overline{X_k}$ could not be allocated on any other PE and the now missing negator hormone would cause the inverse
		assignment task $\overline{X_k}$ resp. $X_k$ to be allocated on its PE instead.
		\item We created one PE each for tasks $A$, $B$ and all join tasks that can only instantiate this specific task.
		Similarly, for each clause $c_i$ in $f$, there is exactly one PE that can only instantiate the literal tasks created from $c_i$'s literals.
		In consequence, a path from $A$ to $B$ exists iff each of those PEs is executing at least one task.
		This allows to check the existence of such path at a glance, especially since the AHS identifies tasks by an integer
		which would otherwise make this check rather complex for the human eye.
	\end{enumerate}
\end{enumerate}

\subsubsection{\NegPath/}
\begin{figure*}
\begin{subfigure}{0.48\linewidth}
\centering
\includegraphics[max width=\linewidth]{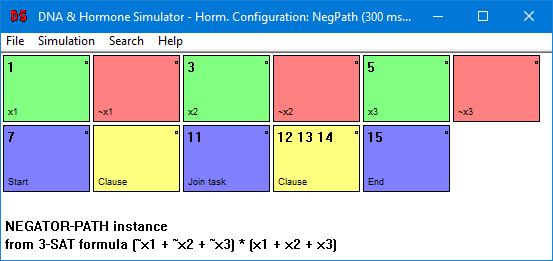}
\caption{Task assignment corresponds to unsatisfying interpretation: $\mathcal{I}(x_1)=\mathcal{I}(x_2)=\mathcal{I}(x_3)=1$}
\label{fig:NegPathExampleSimulatorUnsatisfied}
\end{subfigure}
\hfill
\begin{subfigure}{0.48\linewidth}
\centering
\includegraphics[max width=\linewidth]{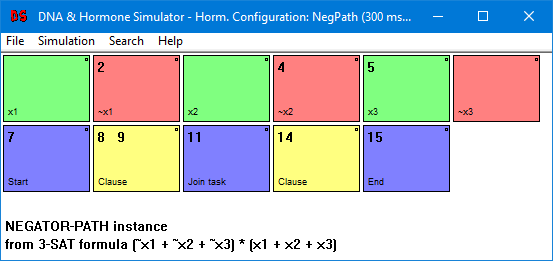}
\caption{Task assignment corresponds to satisfying interpretation: $\mathcal{I}(x_1)=\mathcal{I}(x_2)=0$ and $\mathcal{I}(x_3)=1$}
\label{fig:NegPathExampleSimulatorSatisfied}
\end{subfigure}
\caption{\NegPath/ instance from Figure~\ref{fig:NegSatExampleInterpreted} in AHS simulator. A path from task $A=7$ to $B=15$ exists iff at least one task is running on PE in the second row.}
\label{fig:NegPathExampleSimulator}
\end{figure*}

Figure~\ref{fig:NegPathExampleSimulator} shows the \NegPath/ construction from Figure~\ref{fig:NegPathExampleConstruction} in the simulator for the task assignments from Figure~\ref{fig:NegPathExampleInterpreted}.
Each colored box represents a PE while an integer $i$ denotes that task $i$ is running on the corresponding PE.
The top row of PEs is responsible for instantiating the assignment tasks $1$ to $6$
while the PEs on the bottom row are responsible for instantiating the tasks $A=7$, $B=15$, the join task $11$
as well as all literal tasks.
Thus, a path from $A$ to $B$ exists iff there is at least one task running per bottom-row-PE:
This is false for the unsatisfying interpretation in Figure~\ref{fig:NegPathExampleSimulatorUnsatisfied}, but true for the satisfying one
in Figure~\ref{fig:NegPathExampleSimulatorSatisfied}.

As a consequence, only the second task assignment allows task $A$ to send a message (via multiple hops that might transform it along the way) to task $B$.

\subsubsection{\NegSat/}
\begin{figure*}
\begin{subfigure}{0.48\linewidth}
\centering
\includegraphics[max width=\linewidth]{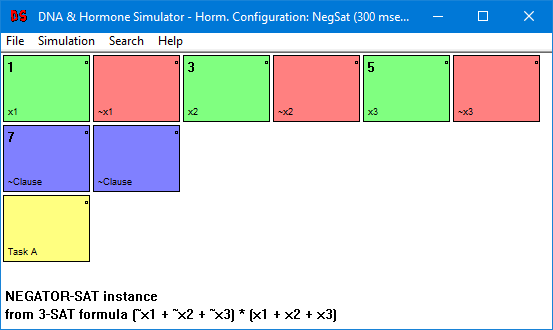}
\caption{Task assignment corresponds to unsatisfying interpretation: $\mathcal{I}(x_1)=\mathcal{I}(x_2)=\mathcal{I}(x_3)=1$}
\label{fig:NegSatExampleSimulatorUnsatisfied}
\end{subfigure}
\hfill
\begin{subfigure}{0.48\linewidth}
\centering
\includegraphics[max width=\linewidth]{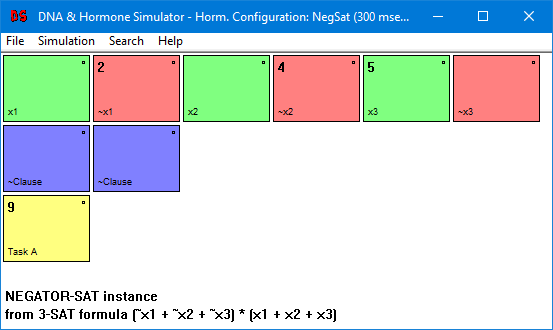}
\caption{Task assignment corresponds to satisfying interpretation: $\mathcal{I}(x_1)=\mathcal{I}(x_2)=0$ and $\mathcal{I}(x_3)=1$}
\label{fig:NegSatExampleSimulatorSatisfied}
\end{subfigure}
\caption{\NegSat/ instance from Figure~\ref{fig:NegSatExampleInterpreted} in AHS simulator}
\label{fig:NegSatExampleSimulator}
\end{figure*}

Figure~\ref{fig:NegSatExampleSimulator} shows the \NegSat/ construction from Figure~\ref{fig:NegSatExampleConstruction} in the simulator for the same task assignments as shown in Figure~\ref{fig:NegSatExampleInterpreted}.

The first row of PEs is again responsible for instantiating the assignment tasks $1$ to $6$ while the blue PEs on the second row are responsible
for instantiating the inverted clause tasks $7$ and $8$.
Finally, the yellow PE in the third row is responsible for instantiating task $A=9$.
As a result, the interpretation corresponding to the task assignment satisfies the formula $f$ iff task $9$ is running on the yellow PE:
This is false for the unsatisfying interpretation in Figure~\ref{fig:NegSatExampleSimulatorUnsatisfied}, but true for the satisfying one
in Figure~\ref{fig:NegSatExampleSimulatorSatisfied}.
It can also easily be seen from example that the inverted clause tasks are assigned iff none of the assignment tasks corresponding to their clause's literals are assigned
and task $A$ is assigned iff none of the inverted clause tasks are assigned.

\subsubsection{\NegStab/}
\begin{figure*}
\begin{subfigure}{0.48\linewidth}
\centering
\includegraphics[max width=\linewidth]{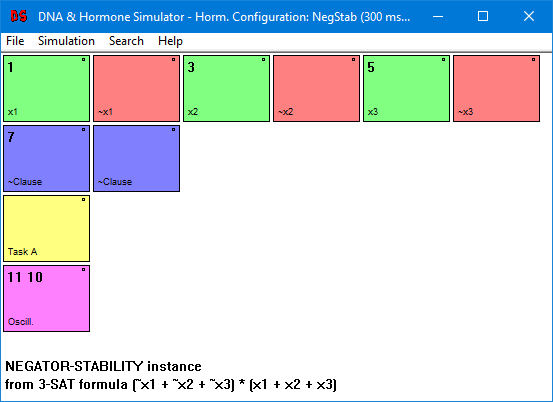}
\caption{Task assignment corresponds to unsatisfying interpretation: $\mathcal{I}(x_1)=\mathcal{I}(x_2)=\mathcal{I}(x_3)=1$}
\label{fig:NegStabExampleSimulatorUnsatisfied}
\end{subfigure}
\hfill
\begin{subfigure}{0.48\linewidth}
\centering
\includegraphics[max width=\linewidth]{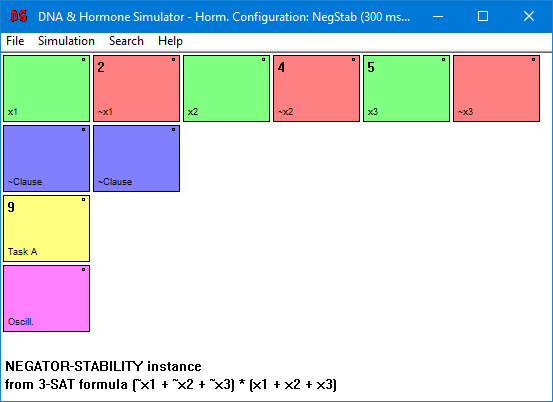}
\caption{Task assignment corresponds to satisfying interpretation: $\mathcal{I}(x_1)=\mathcal{I}(x_2)=0$ and $\mathcal{I}(x_3)=1$}
\label{fig:NegStabExampleSimulatorSatisfied}
\end{subfigure}
\caption{\NegStab/ instance for $f=(\overline{x_1}\lor\overline{x_2}\lor\overline{x_3})\land (x_1\lor x_2\lor x_3)$ in AHS simulator}
\label{fig:NegStabExampleSimulator}
\end{figure*}

Figure~\ref{fig:NegStabExampleSimulator} shows a \NegStab/ instance constructed from formula $f$ in the simulator.
The first three rows of PEs are identical to the \NegSat/ instance from Figure~\ref{fig:NegSatExampleSimulator},
the PE in the fourth row is responsible for instantiating the three oscillating tasks $X=10$, $Y=11$ and $Z=12$ (cf. Figure~\ref{fig:NegStabConstruction}).

In case of the unsatisfying interpretation shown in Figure~\ref{fig:NegStabExampleSimulatorUnsatisfied}, task $A$ is not assigned and thus tasks $10$, $11$ and $12$ cause the assignment to oscillate.
In fact, the snapshot shows both $11$ and $10$ assigned, corresponding to Figure~\ref{fig:NegStabOscillationTriangle}, sequence~(3).

In contrast, Figure~\ref{fig:NegStabExampleSimulatorSatisfied} shows a satisfying interpretation of $f$.
Here, task $A=9$ can be assigned and no task can be instantiated on the PE in the fourth row, causing the task assignment to be stable.

\subsubsection{Summary}
These examples show that the theoretical considerations regarding the problems \NegPath/, \NegSat/ and \NegStab/ can also be applied to a real implementation of our negator concept.
Thus, these problems are highly relevant when designing systems using negator relationships:
If a message can never reach its destination (no \NegPath/ instance), some task cannot be instantiated in a stable system at all (no \NegSat/ instance) or the system is not stable (no \NegStab/ instance), the resulting system most probably contains some design mistake.

The next section will deal with another aspect of negators that explains \emph{why} problems involving negator relationships are computationally hard.

%% file: LogicGates.tex
\section{Constructing Logic Circuits using Negators}
\label{sec:boolLogicWithNegators}
In the previous sections, we have seen that various negator-related problems are computationally hard.
In order to explain \emph{why} this is the case, we will now look at these problems from a different perspective.
To do so, we will primarily generalize the concepts used for the reduction from \dreisat/ to \NegSat/ as shown in \cite{HutterSENSYBLE2020} (cf. Figure~\ref{fig:NegSatExampleConstruction}).
Nevertheless, the results are also applicable to \NegPath/ and \NegStab/.

\subsection{Interpreting Negator Hormones as Logic Levels}
In the following, we will interpret negator hormones and the tasks sending them as boolean logic levels.
Specifically, the following convention will be used:
\begin{itemize}
	\item HIGH $\iff$ negator hormone is sent\newline
		\phantom{HIGH} $\iff$ sending task is assigned,
	\item \phantomas[l]{HIGH}{LOW} $\iff$ negator hormone is \emph{not} sent\newline
		\phantom{HIGH} $\iff$ sending task is \emph{not} assigned.
\end{itemize}
This interpretation allows to construct arbitrary logic gates using negators.

\begin{figure}
\centering
\includegraphics[page=1]{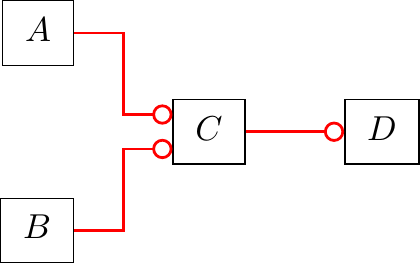}
\caption{Task set with negators behaving like a NOR gate}
\label{fig:NORWithNegators}
\end{figure}

\begin{table}
	\begin{subtable}[t]{0.48\linewidth}
		\centering
		\begin{tabular}{cc|c}
			\toprule
			$A$ & $B$ & \boldmath$C$\unboldmath \\
			\midrule
			0 & 0 & \textbf{1} \\
			0 & 1 & \textbf{0} \\
			1 & 0 & \textbf{0} \\
			1 & 1 & \textbf{0} \\
			\bottomrule
		\end{tabular}
		\caption{Assignment of task $C$ in dependence of $A$ and $B$'s assignment (1 $\iff$ task is assigned)}
		\label{tbl:Assignment}
	\end{subtable}\hfill%
	\begin{subtable}[t]{0.48\linewidth}
		\centering
		\begin{tabular}{cc|c}
			\toprule
			$a$ & $b$ & \boldmath$a \overline{\lor} b$\unboldmath \\
			\midrule
			0 & 0 & \textbf{1} \\
			0 & 1 & \textbf{0} \\
			1 & 0 & \textbf{0} \\
			1 & 1 & \textbf{0} \\
			\bottomrule
		\end{tabular}
		\caption{Truth table of NOR gate}
		\label{tbl:NORGate}
	\end{subtable}
	\caption{Assignment of task $C$ compared to truth table of NOR gate}
\end{table}

Consider the task set with negator relationships shown in Figure~\ref{fig:NORWithNegators}.
The negator hormone from task $C$ to $D$ will be sent iff $C$ is assigned to some PE.
However, $C$ will only be assigned if no negator hormone is sent to $C$.
Thus, the negator hormone from $C$ to $D$ will only be sent if neither $A$ nor $B$ send a negator hormone to $C$.
Equivalently, $C$ will only be assigned if neither $A$ nor $B$ are assigned.
$C$'s assignment status in dependence of $A$ and $B$'s assignment states is shown in Table~\ref{tbl:Assignment}.
However, when compared to Table~\ref{tbl:NORGate}, it can be seen that $C$'s assignment matches the formula $c=a\, \overline{\lor}\, b$ and thus a NOR gate's behavior.

Since NOR on its own forms a functionally complete operator set, arbitrary logic gates can be constructed using the NOR gate as a building block.

\subsection{CNF to Task Set with Negators}
A formula $f$ in conjunctive normal form (CNF) can easily be represented as a task set with negator relationships
by realizing it as a two-stage NOR circuit.
This can be achieved by double negation and application of De Morgan's laws:
\begin{align*}
	f &= (\overline{x_1}\lor \overline{x_2} \lor \overline{x_3}) \land (x_1 \lor x_2 \lor x_3) \\
	&\equiv \overline{\overline{(\overline{x_1}\lor \overline{x_2} \lor \overline{x_3}) \land (x_1 \lor x_2 \lor x_3)}} \\
	&\equiv \overline{\overline{(\overline{x_1}\lor \overline{x_2} \lor \overline{x_3})} \lor \overline{(x_1 \lor x_2 \lor x_3)}}
\end{align*}
Figure~\ref{fig:CnfNorExample} shows the graphical construction of $f$ as a two-stage NOR circuit and the conversion to a task set with negators using Figure~\ref{fig:NORWithNegators} as a building block:
Task $\overline{C_1}$ and $\overline{C_2}$ can only be assigned if \emph{none} of their respective literal tasks are assigned;
task $A$ can only be assigned if neither $\overline{C_1}$ nor $\overline{C_2}$ are assigned.

\begin{figure*}
	\centering
	\begin{subfigure}[t]{0.3\linewidth}
	\centering
	\includegraphics[max width=\linewidth,page=2]{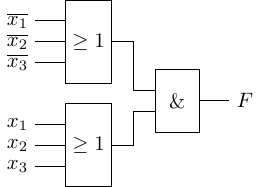}
	\caption{Two-stage OR/AND realization}
	\end{subfigure}
	\hfill
	\begin{subfigure}[t]{0.3\linewidth}
	\centering
	\includegraphics[max width=\linewidth,page=3]{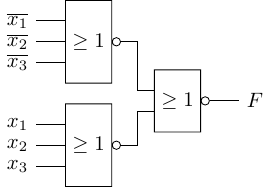}
	\caption{After double negation and De Morgan's laws: Two-stage NOR realization}
	\end{subfigure}
	\hfill
	\begin{subfigure}[t]{0.3\linewidth}
	\centering
	\includegraphics[max width=\linewidth,page=4]{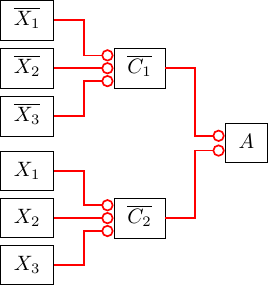}
	\caption{Converted to task set with negators}
	\label{fig:CnfNorExample:Negators}
	\end{subfigure}
	\caption{Construction of two-stage NOR circuit from CNF formula and conversion to task set with negators}
	\label{fig:CnfNorExample}
\end{figure*}

In fact, this construction has already been used implicitly for transforming \dreisat/ to \NegSat/ in Theorem~\ref{thm:NegSatNPComplete}'s proof (cf. \cite{HutterSENSYBLE2020} and Figure~\ref{fig:NegSatExampleConstruction}), albeit the transformation used
an additional building block to ensure that exactly one of the tasks $x_k$ and $\overline{x_k}$ is assigned at any time, thus forcing a consistent formula interpretation by means of mutual exclusion.
As this building block is missing from the construction shown in Figure~\ref{fig:CnfNorExample:Negators}, the assignment tasks could be assigned in a conflicting way.

\subsection{Flip-Flops with Negators}
In fact, the mentioned mutual exclusion building block (as shown in Figure~\ref{fig:AssignmentTasks}) behaves like a SR latch built from NOR gates.
Figure~\ref{fig:SRlatchNOR} shows such a SR latch while Figure~\ref{fig:SRlatchNegators} depicts a task set with negators behaving identically.
This is achieved by using the mutual exclusion building block from Figure~\ref{fig:AssignmentTasks} resp. the NOR building block from Figure~\ref{fig:NORWithNegators}:
In a stable system, either task $Q$ or task $\overline{Q}$ may be assigned while the other task is blocked by the sent negator.
By temporarily sending a negator ($R$ resp. $S$) to the assigned task, the latch's current value can be flipped.

\begin{figure*}
	\centering
	\begin{subfigure}[t]{0.3\linewidth}
	\centering
	\includegraphics[max width=\linewidth,page=5]{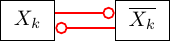}
	\caption{Building block: Assignment tasks realizing mutual exclusion of tasks $x_k$ and $\overline{x_k}$}
	\label{fig:AssignmentTasks}
	\end{subfigure}
	\hfill
	\begin{subfigure}[t]{0.3\linewidth}
	\centering
	\includegraphics[max width=\linewidth,page=6]{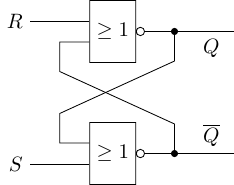}
	\caption{SR latch built from NOR gates}
	\label{fig:SRlatchNOR}
	\end{subfigure}
	\hfill
	\begin{subfigure}[t]{0.3\linewidth}
	\centering
	\includegraphics[max width=\linewidth,page=7]{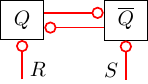}
	\caption{Task set with negators behaving like a SR latch}
	\label{fig:SRlatchNegators}
	\end{subfigure}
	\caption{Construction of flip-flops using task sets with negators}
\end{figure*}

In theory, this allows to construct more complex flip-flops (e.g. gated D latches) or even whole sequential circuits by appropriately combining different building blocks although the practical relevance of their \emph{intentional} construction seems questionable at best.

\subsection{Oscillating Task Assignments}
\label{sec:boolLogicWithNegators:Oscillation}
Figure~\ref{fig:Oscillation:Inverter} shows a ring oscillator circuit built from three inverter gates.
Since an inverter can also be considered as a NOR gate with a single input, it is easy to adapt this circuit to negator logic as shown in Figure~\ref{fig:Oscillation:Negator}.

\begin{figure*}
	\centering
	\begin{subfigure}[t]{0.45\linewidth}
	\centering
	\includegraphics[max width=\linewidth,page=9]{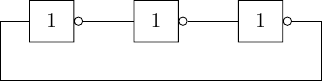}
	\caption{Ring oscillator circuit constructed from an odd number of inverters}
	\label{fig:Oscillation:Inverter}
	\end{subfigure}
	\hfill
	\begin{subfigure}[t]{0.45\linewidth}
	\centering
	\includegraphics[max width=\linewidth,page=8]{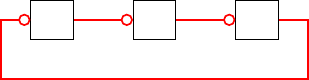}
	\caption{Equivalent realization using a task set with negator relationships}
	\label{fig:Oscillation:Negator}
	\end{subfigure}
	\caption{Construction of a ring oscillator circuit}
	\label{fig:Oscillation}
\end{figure*}

In fact, we have already encountered this gadget in the reduction from \NegSat/ to \NegStab/ (cf. Theorem~\ref{thm:NegStabNPHard}).
The negators' oscillation pattern known from Figure~\ref{fig:NegStabOscillationTriangle} corresponds to the one realized by the three inverters while the pattern from Figure~\ref{fig:NegStabOscillationTriangleVariant2} is a peculiarity of the AHS and its implementation.

As oscillating task sets obviously prevent a system from ever reaching a stable task assignment,
they are yet another pitfall to avoid while employing negator relationships for designing systems.

\subsection{Sequential Logic}
\label{sec:boolLogicWithNegators:Sequential}
In theory, it should even be possible to realize sequential logic circuits with negators using the flip-flop and ring oscillator building blocks.
As an example, consider a \dreisat/ formula $f$ transformed to a \NegSat/ instance (cf. section~\ref{sec:Examples:Theoretical:NegSat}).
\NegSat/ only asks whether a task assignment exists so that the system is stable and some distinguished task $A$ is assigned, but does not require the AHS to find such task assignment on its own.

\begin{figure*}
	\centering
	\includegraphics[width=\linewidth]{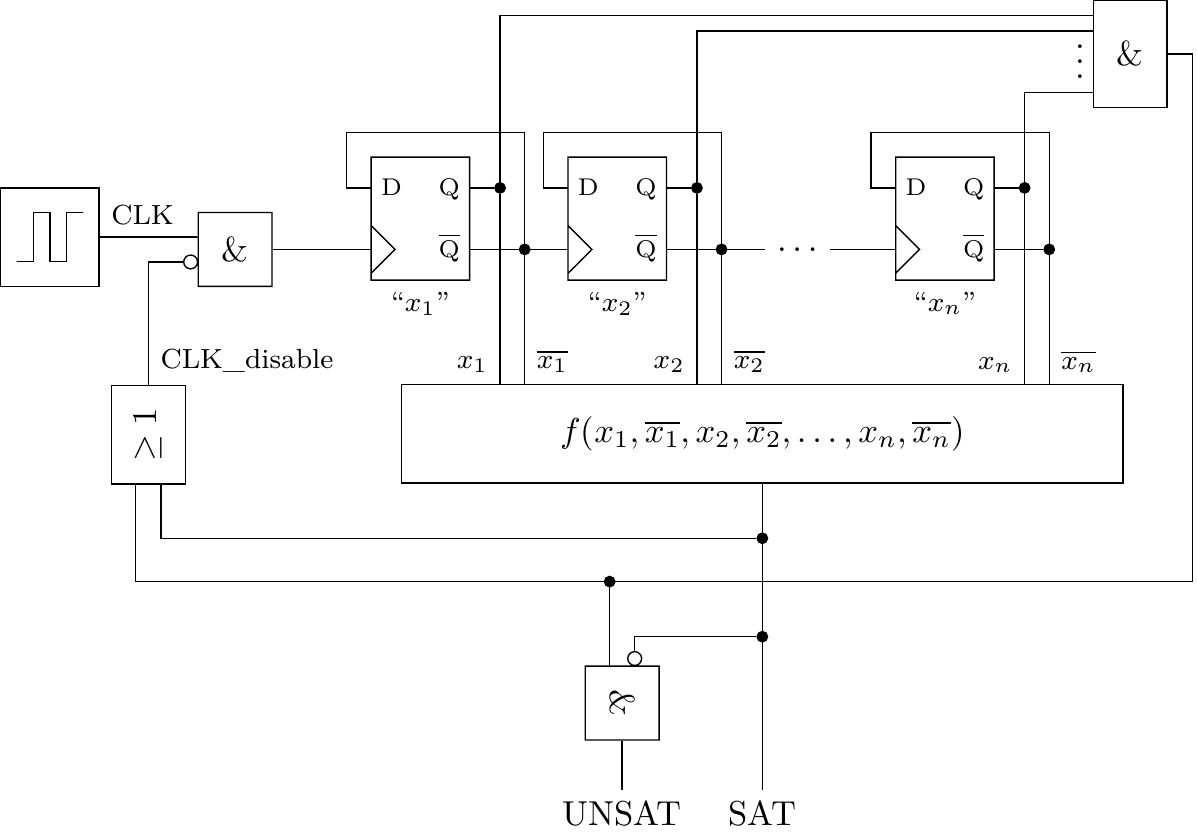}
	\caption{Example of a sequential logic circuit that could be implemented using tasks and negator relationships}
	\label{fig:SequentialSatSolver}
\end{figure*}

Using sequential logic, one could actually realize a brute force SAT solver like the example shown in Figure~\ref{fig:SequentialSatSolver} inside the AHS using tasks and negator relationships.
This circuit uses one flip-flop per variable $x_i$ that is triggered on a rising clock edge.
Starting with $x_1=x_2=\dots=x_n=0$, the circuit sequentially evaluates $f$ for each of the $2^n$ input combinations.
Once either all combinations were tested ($x1=x_2=\dots=x_n=1$) or $f$ evaluates to $1$, the clock is disabled by pulling it to $0$.
One of the lines SAT resp. UNSAT rises to $1$ once the clock is stopped and gives the final result regarding $f$'s satisfiability.

%% file: Conclusion.tex
\section{Conclusion}
\label{sec:Conclusion}
In this report, we revisited the \NegPath/ and \NegSat/ problems originally introduced in \cite{HutterISORC2020} and \cite{HutterSENSYBLE2020}.
First, we presented an alternative proof for \NegPath/'s NP-hardness by utilizing \NegSat/'s NP-completeness.
Then, the new problem \NegStab/ was introduced and also proved to be NP-complete.

In addition, we presented examples of \NegPath/, \NegSat/ and \NegStab/ instances in an AHS simulator.
This proved that the theoretical results are applicable to real-world implementations of the negator concept and thus highly relevant when designing systems using negator relationships.

Moreover, we have shown building blocks composed from tasks and negator relationships that behave like NOR gates and flip-flops and even ring oscillators.
In theory, combining these building blocks allows the construction of sequential logic circuits using tasks and negator relationships.
Of course, the practical relevance of constructing arbitrary logic circuits with negators is questionable to say the least:
Using tasks and negators to build a circuit and evaluating it by (ab-)using the AHS implies a \emph{massive} resource and performance overhead while an equivalent function could simply be realized as a single task computing the result.

Nevertheless, this is an interesting theoretical result as it again shows the power introduced by negators and especially answers \emph{why} seemingly simple questions such as
\begin{itemize}
	\item ``Does a stable task assignment exist so that task $A$ may send a message to $B$ (via multiple hops)?'' (\NegPath/),
	\item ``Does a stable task assignment exist that includes task $A$?'' (\NegSat/) or even
	\item ``Does a stable task assignment exist?'' (\NegStab/)
\end{itemize}
are hard to answer when dealing with negator relationships:
Negators can be used to represent arbitrary logic circuits, thus explaining why the reduction of \dreisat/ to the mentioned problems was successful.

In future work, we plan to conduct research on additional aspects of negators.
Since even checking whether a stable task assignment exists for a given task set at all, it is an interesting question if the use of negators can be restricted in a way that still allows to express sufficiently powerful task relationships while lowering the complexity of various decision problems.

This is also a prerequisite for investigating the timing aspects of negators.
Since the original AHS (without negators) guarantees tight bounds for the duration of its self-configuration and self-healing after PE failures, it is an interesting question which timing guarantees can be given when negators are present.
Of course, this requires a stable system or its self-configuration's duration cannot be bounded at all.